\documentclass{article}

\usepackage{amsmath,amssymb,graphicx,epstopdf,color}
\usepackage{url,algorithm,algorithmic}

\usepackage{amsthm,vmargin}

\setpapersize{A4}

\newcommand{\g}[1]{\boldsymbol{#1}}

\newcommand{\I}[1]{\mathbf{1}_{#1}}

\renewcommand{\H}[0]{\mathcal{H}} 
 
\newcommand{\X}[0]{\mathcal{X}} 
\newcommand{\Y}[0]{\mathcal{Y}} 
\newcommand{\Q}[0]{\mathcal{Q}} 
\newcommand{\QR}[0]{\mathbb{Q}} 

\renewcommand{\O}[0]{\mathcal{O}}

\renewcommand{\d}[0]{d} 
\newcommand{\sign}[0]{\mbox{sign}} 
\newtheorem{theorem}{Theorem}
\newtheorem{problem}{Problem}
\newtheorem{assumption}{Assumption}
\newtheorem{definition}{Definition}
\newtheorem{proposition}{Proposition}

\newtheorem{corollary}{Corollary}
\newenvironment{proofsketch}{{\em Proof sketch.}}{ \hfill\qed }

\newcommand{\argmax}{\operatornamewithlimits{argmax}}
\newcommand{\argmin}{\operatornamewithlimits{argmin}}

\begin{document}

\title{\bf On the complexity of \\ piecewise affine system identification} 
\author{Fabien Lauer\medskip\\\small Universit\'e de Lorraine, CNRS, LORIA, UMR 7503, F-54506 Vand\oe{}uvre-l\`es-Nancy, France}
\maketitle
\begin{abstract}
The paper provides results regarding the computational complexity of hybrid system identification. More precisely, we focus on the estimation of piecewise affine (PWA) maps from input-output data and analyze the complexity of computing a global minimizer of the error. Previous work showed that a global solution could be obtained for continuous PWA maps with a worst-case complexity exponential in the number of data. In this paper, we show how global optimality can be reached for a slightly more general class of possibly discontinuous PWA maps with a complexity only polynomial in the number of data, however with an exponential complexity with respect to the data dimension. This result is obtained via an analysis of the intrinsic classification subproblem of associating the data points to the different modes. In addition, we prove that the problem is NP-hard, and thus that the exponential complexity in the dimension is a natural expectation for any exact algorithm.
\end{abstract}

\section{Introduction}
\label{sec:intro}

Hybrid system identification aims at estimating a model of a system switching between different operating modes from input-output data. More precisely, most of the literature considers autoregressive with external input (ARX) models to cast the problem as a regression one \cite{Paoletti07}. Then, two cases can be distinguished: switching regression, where the system arbitrarily switches from one mode to another, and piecewise affine (PWA) regression, where the switches depend on the regressors. 
A number of methods with satisfactory performance in practice are now available for these problems \cite{Garulli12}. However, compared with linear system identification, a major weakness of these methods is their lack of guarantees. 

For the particular case of noiseless data, the algebraic method \cite{Vidal03} provides a solution to switching regression with a small number of modes. However, the quality of the estimates quickly degrades with the increase of the noise level. A few sparsity-based methods \cite{Bako11,Maruta14} also offer guarantees in the noiseless case, but these are subject to a condition on both the data and the sought solution.
In the presence of noise, most methods consider the minimization of the error of the model over the data \cite{Paoletti07}. While this does not necessarily yields the best predictive model (due to issues like identifiability, persistence of excitation and access to a limited amount of data), obtaining statistical guarantees with such an approach has a long history in statistics and system identification \cite{Ljung99}. 
However, such results are not available for hybrid systems. This is probably due to the fact that minimizing the error of a hybrid model is a difficult nonconvex optimization problem involving the simultaneous classification of the data points into modes and the regression of a submodel for each mode. Thus, theoretical guarantees could only be obtained under the rather strong assumption that this problem has been  solved to global optimality and most of the literature \cite{Ferraritrecate03,Bemporad05,Juloski05b,Lauer11a,Lauer13a,Lauer14a} focuses on this issue with heuristics of various degrees of accuracy and computational efficiency. 
Many recent works \cite{Bako11,Ozay12,Ohlsson10,Ohlsson13,Lauer12a,Maruta14} try to avoid local minima by considering convex formulations, but these only yield optimality with respect to a relaxation of the original problem. 
Global optimality in the presence of noise was only reached in \cite{Roll04} for a particular class of continuous PWA maps known as hinging-hyperplanes by reformulating the problem as a mixed-integer program solved by branch-and-bound techniques. However, such optimization problems are NP-hard \cite{Garey79} and branch-and-bound algorithms have a worst-case complexity exponential in the number of integer variables, here proportional to the number of data and the number of modes.

Inspired by related clustering problems, such as the minimization of the sum of squared distances between points and their group centers, we could minimize the hybrid model error by enumerating all possible classifications of the points. But the number of classifications is exponential in the number of data. Conversely, the other approach enumerating a sample of values for the real variables of the problem is exponential in the dimension and can only offer an approximate solution. 

Overall, the literature does not provide a method that can guarantee both the optimality and the computability of a global minimizer of the error, while the computational complexity of this problem remains unknown and cannot be deduced from the NP-hardness of classical clustering problems~\cite{Aloise09} 
(see~\cite{Garey79,Blondel00} for an introduction to computational complexity and its relevance to control theory).



\paragraph*{Contribution}
The paper provides two results regarding the computational complexity of PWA regression, and more precisely for the problem of finding a global minimizer of the error of a PWA model, formalized in Sect.~\ref{sec:problem}. First, we show in Sect.~\ref{sec:NPhard} that the problem is NP-hard. 
Then, we show in Sect.~\ref{sec:exact} that, for any fixed dimension of the data, an exact solution can be computed in time polynomial in the number of data via an enumeration of all possible classifications.  
To obtain this result and avoid the exponential growth of the number of classifications with the number of data, 
we show that, in PWA regression, the classification of the data points is highly constrained and the number of classifications to test can be limited.
The price to pay for this gain is an exponential complexity with respect to the data dimension and the number of modes. 
Future work is outlined in Sect.~\ref{sec:discuss}.



\paragraph*{Notations}
We use the indicator function $\I{E}$ of an event $E$ that is $1$ if the event occurs and $0$ otherwise. We define $\sign(u) = 1$ if $u\geq0$ and $-1$ otherwise. Given a set of labels $\Q\subset\mathbb{Z}$ and a  set of $N$ points, a labeling of these points is any $\g q \in\Q^N$. We use $j=\argmax_{k\in\Q} u(k)$ as a shorthand for $j=\min \{l \in \argmax_{k\in\Q} u(k)\}$. Given two sets, $\X$ and $\Y$, $\Y^{\X}$ is the set of functions from $\X$ to~$\Y$.

\section{Problem formulation}
\label{sec:problem}

As in most works, we concentrate on discrete-time PWARX system identification considered as a PWA regression problem with regression vectors $\g x_i = [y_{i-1}, \dots,y_{i-n_y}, u_i, \dots, u_{i-n_u}]^T\in\X$ built from past inputs $u_{i-k}$ and outputs $y_{i-k}$. 
Since we are interested in computational complexity results, we restrain the data to rational, digitally representable, values and set $\X\subseteq \QR^d$. 
The outputs are assumed to be generated by a PWA system $f$ as 
$y_i = f(\g x_i) + v_i$,
where $v_i$ is a noise term. More precisely, PWA models can be expressed via a set of $n$ affine submodels and a function $h:\X\rightarrow  \Q=\{1,\dots,n\}$ determining the active submodel: $f(\g x) = \g w_{h(\g x) }^T\overline{\g x}$, where $\overline{\g x} = [\g x^T, 1]^T$. 

We call the function $h$ a classifier as it classifies the data points in the different modes.
Typically, PWA systems are defined with $h$ implementing a polyhedral partition of $\X$, with modes possibly spanning unions of polyhedra. However, in most of the literature on PWA system identification \cite{Paoletti07,Ferraritrecate03,Bemporad05,Juloski05b,Lauer12a}, $h$ is estimated within the family of linear classifiers
\begin{equation}\label{eq:hmulticlass}
	\H = \{ h \in \Q^\X : h(\g x) = \argmax_{k\in\Q} \g h_k^T\g x + b_k,\ \g h_k\in \QR^d, \ b_k\in\QR \} ,
\end{equation}
based on a set of $n$ linear functions and for which a mode spanning a union of polyhedra must be modeled as several modes with similar affine submodels. 
For PWA maps with $n=2$ modes, $h$ is a binary classifier for which it is common to consider its output in $\Q=\{-1,+1\}$ instead of $\{1,2\}$. Such a binary classifier can be obtained by taking the sign of a real-valued function. 
If this function is linear (or affine), then we obtain a linear classifier, which is equivalent to a separating hyperplane dividing the input space $\X$ in two half-spaces. In this case, the function class $\H$ can be defined as
\begin{equation}\label{eq:Hbinary}
	\H = \{ h \in \Q^\X\ :\ h(\g x) = \sign(\g h^T\g x + b),\ \g h\in \QR^d, \ b\in\QR \} 
\end{equation}
with a single set of parameters $(\g h,b)$ corresponding to the normal to the hyperplane and the offset from the origin. An equivalence with the multi-class formulation in~\eqref{eq:hmulticlass} is obtained by using $\g h = \g h_1 - \g h_2$ and $b=b_1-b_2$.

In this paper, we consider the common estimation approach of minimizing the error on $N$ data pairs $(\g x_i, y_i)\in\X\times \QR$, measured pointwise by a loss function $\ell : \QR\rightarrow \QR^+$ as
$$
	\ell(y_i- f(\g x_i) ) = \sum_{j\in\Q} \I{h(\g x_i) = j}\ \ell(y_i- \g w_j^T\overline{\g x}_i ) .
$$
More precisely, we focus on well-posed instances of the problem where $N$ is significantly larger than the dimension $d$ and the number of modes $n$ is given. Indeed, with free $n$ the problem is ill-posed as the solution is only defined up to a trade-off between the number of modes and the model accuracy. For the converse well-posed approach that minimizes $n$ for a given error bound, a complexity analysis can be found in~\cite{Alur14}.
Under these assumptions, the problem is as follows.
\begin{problem}[Error-minimizing PWA regression]\label{pb:PWA}
Given a data set $\{(\g x_i, y_i)\}_{i=1}^N \in( \X\times \QR)^N$ with $\X\subseteq\QR^d$ and an integer $n\in[2,N/(d+1)]$, find a global solution to 
\begin{align}
	&\min_{ \g w\in\QR^{n(d+1)}, h\in\H }\ \frac{1}{N} \sum_{i=1}^N \sum_{j\in\Q} \I{h(\g x_i) = j}\ \ell(y_i- \g w_j^T\overline{\g x}_i ) ,
\end{align}
where $\g w = (\g w_j)_{j\in \Q}$ is the concatenation of all parameter vectors and $\H \subset \Q^\X$ is the set of $n$-category linear classifiers as in~\eqref{eq:hmulticlass} or~\eqref{eq:Hbinary}.
\end{problem}

The following analyzes the time complexity of Problem~\ref{pb:PWA} under the classical model of computation known as a Turing machine \cite{Garey79}. The time complexity of a problem is the lowest time complexity of an algorithm solving any instance of that problem, where the time complexity of an algorithm is the maximal number of steps occuring in the computation of the corresponding Turing machine program. The loss function $\ell$ is assumed to be computable in polynomial time throughout the paper.



\section{NP-hardness}
\label{sec:NPhard}

This section contains the proof of the following NP-hardness result, where an NP-hard problem is one that is at least as hard as any problem from the class NP of nondeterministic polynomial time decision problems~\cite{Garey79} (NP is the class of all decision problems for which a solution can be certified in polynomial time). 
\begin{theorem}\label{thm:NPhard}
With a loss function $\ell$ such that $\ell(e)=0 \Leftrightarrow e=0$, Problem~\ref{pb:PWA} is NP-hard.
\end{theorem}

The proof uses a reduction from the partition problem, known to be NP-complete~\cite{Garey79}, i.e., a problem that is both NP-hard and in NP. 
\begin{problem}[Partition]\label{pb:partition}
Given a multiset (a set with possibly multiple instances of its elements) of $d$ positive integers, $S= \{s_1,\dots,s_d\}$, decide 
whether there is a multisubset $S_1\subset S$ such that
$$
	\sum_{s_i \in S_1} s_i = \sum_{s_i\in S\setminus S_1} s_i .
$$
\end{problem}

More precisely, we will reduce Problem~\ref{pb:partition} to the decision form of Problem~\ref{pb:PWA}. 
\begin{problem}[Decision form of PWA regression]\label{pb:PWAdecision}
Given a data set $\{(\g x_i, y_i)\}_{i=1}^N \in( \X\times \QR)^N$, an integer $n\in[2,N/(d+1)]$ and a threshold $\epsilon\geq 0$, decide whether there is a pair $(\g w, h) \in\QR^{n(d+1)}\times \H$ such that
\begin{equation}\label{eq:decision}
	\frac{1}{N} \sum_{i=1}^N \sum_{j\in\Q} \I{h(\g x_i) = j}\ \ell(y_i- \g w_j^T\overline{\g x}_i ) \leq \epsilon ,
\end{equation}
where $\H$ is the set of linear classifiers as in~\eqref{eq:hmulticlass} or~\eqref{eq:Hbinary} and the loss function $\ell$ is such that $\ell(e)=0 \Leftrightarrow e=0$.
\end{problem}

\begin{proposition}
Problem~\ref{pb:PWAdecision} is NP-complete.
\end{proposition}
\begin{proof}
Since given a candidate pair $(\g w, h)$ the condition~\eqref{eq:decision} can be verified in polynomial time, Problem~\ref{pb:PWAdecision} is in NP. Then, the proof of its NP-completeness proceeds by showing that the Partition Problem~\ref{pb:partition} has an affirmative answer if and only if Problem~\ref{pb:PWAdecision} with $\epsilon=0$ has an affirmative answer. 

Given an instance of Problem~\ref{pb:partition}, let $N=2d+3$, $n=2$, $\Q=\{-1,1\}$ and build a data set with
$$
	(\g x_i,\ y_i) = \begin{cases}
		(s_i \g e_i,\ s_i) ,& \mbox{if } 1\leq i \leq d \\
		(-s_{i-d} \g e_{i-d},\ s_{i-d}) ,& \mbox{if } d< i \leq 2d \\
		(\g s,\ 0), & \mbox{if } i=2d+1\\
		(-\g s,\ 0), & \mbox{if } i=2d+2\\
		(\g 0,\ 0) , & \mbox{if } i=2d+3,

	\end{cases}
$$
where $\g e_k$ is the $k$th unit vector of the canonical basis for $\QR^d$ and $\g s = \sum_{k=1}^d s_k \g e_k$. 
If Problem~\ref{pb:partition} has an affirmative answer, then, using the notations of~\eqref{eq:Hbinary}, we can set
$$
	\g w_1 = \sum_{k\in I_1} \overline{\g e}_k - \sum_{k\in I_{-1}} \overline{\g e}_k,\ \g w_{-1} = -\g w_1,\quad
	\g h = \sum_{k\in I_1} \g e_k - \sum_{k\in I_{-1}} \g e_k,\ b=0,
$$
where $\overline{\g e}_k=[\g e_k^T,\ 0]^T$, $I_1$ is the set of indexes of the elements of $S$ in $S_1$ and $I_{-1} = \{1,\dots,d\}\setminus I_1$. 
This gives 
$$
	\g w_1^T\overline{\g x}_i = \begin{cases}
		s_i = y_i ,& \mbox{if } i\leq d \mbox{ and } i \in I_1 \\
		-s_i ,& \mbox{if } i\leq d \mbox{ and } i\in I_{-1} \\
		s_{i-d} = y_i ,& \mbox{if } i > d \mbox{ and } i-d \in I_{-1} \\		
		-s_{i-d} ,& \mbox{if } i > d \mbox{ and } i-d \in I_1 \\
		\sum_{k\in I_1} s_k - \sum_{k\in I_{-1}} s_k = 0=y_i, & \mbox{if } i=2d+1\\
		\sum_{k\in I_{-1}} s_k -\sum_{k\in I_1} s_k = 0=y_i, & \mbox{if } i=2d+2\\
		0  = y_i , & \mbox{if } i=2d+3
	\end{cases}
$$
and we can similarly show that 
$$
	\g w_{-1}^T \overline{\g x}_i  = y_i, \quad  \mbox{if } i \in I_{-1} \vee i-d \in I_1 \vee i > 2d,
$$
while $\g h^T \g x_i$ is positive if $i\in I_1 \vee i-d \in I_{-1}$ and negative if $i\in I_{-1} \vee i-d \in I_1$. Therefore, for all points, $\g w_{h(\g x_i)}^T\overline{\g x}_i = y_i$, $i=1,\dots,2d+3$, and the cost function of Problem~\ref{pb:PWA} is zero, yielding an affirmative answer for Problem~\ref{pb:PWAdecision}. 

It remains to prove that if~\eqref{eq:decision} holds with $\epsilon=0$, then Problem~\ref{pb:partition} has an affirmative answer. To see this, note that due to $\ell$ being positive, a zero cost implies a zero loss for all data points. Thus, by $\ell(e)=0 \Leftrightarrow e=0$, if~\eqref{eq:decision} holds with $\epsilon=0$,
\begin{equation}\label{eq:perfectfit}
	\g w_{h(\g x_i)}^T\overline{\g x}_i  = y_i,\quad i=1,\dots,2d+3.
\end{equation}
Also note that if $h(\g x_i) = h(\g x_{i+d}) = 1$ for some $i\leq d$, we have $s_i w_{1,i} + w_{1,d+1}= -s_i w_{1,i}+ w_{1,d+1}= s_i$. This is only possible if $s_i=0$, which is not the case (otherwise we can simply remove $s_i$ without influencing the partition problem), or if $w_{1,d+1} = s_i$. The latter is impossible if $h(\g x_i) = h(\g x_{i+d})$ since $h$ is a linear classifier that must return the same category for all points on the line segment between $\g x_i$ and $\g x_{i+d}$, which includes the origin $\g x_{2d+3} = \g 0$ and thus would imply by~\eqref{eq:perfectfit} that $\g w_1^T \overline{\g x}_{2d+3} = w_{1,d+1}= y_{2d+3}=0$. As a consequence, $h(\g x_i) = h(\g x_{i+d}) = 1$ cannot hold, and since we can similarly show that $h(\g x_i) = h(\g x_{i+d}) = -1$ cannot hold, we have $h(\g x_i) \neq h(\g x_{i+d})$ for all $i\leq d$. Hence,~\eqref{eq:perfectfit} leads to
\begin{align}
	&\g w_1^T\overline{\g x}_i \neq y_i \quad  \Rightarrow \quad \g w_{1} ^T\overline{\g x}_{d+i} = y_{d+i},\quad i=1,\dots,d \label{eq:nphardimplication1}\\
	&\g w_{-1}^T\overline{\g x}_i \neq y_i \quad  \Rightarrow \quad \g w_{-1} ^T\overline{\g x}_{d+i} = y_{d+i},\quad i=1,\dots,d. \label{eq:nphardimplication2}
\end{align}
Let $\hat{I}_1 = \{ i\in\{1,\dots, d\} :  \g w_1^T\overline{\g x}_i = y_i \}$ and $\hat{I}_{-1} =  \{1,\dots, d\} \setminus \hat{I}_1$. Then, if $h(\g x_{2d+3})= +1$, $w_{1,d+1}=0$ and for all $i\leq d$, $\g w_1^T\overline{\g x}_i = w_i s_i$. Therefore, for all $i\in \hat{I}_1$, $w_i = 1$, while for all $i\in \hat{I}_{-1}$, \eqref{eq:nphardimplication1} gives $\g w_{1} ^T\overline{\g x}_{d+i} = y_{d+i}$, i.e., $-w_i s_i = s_i$ and $w_i = -1$. This leads to
\begin{align*}
	&\g w_1^T \overline{\g x}_{2d+1} = \sum_{i\in \hat{I}_1  } s_i - \sum_{i\in \hat{I}_{-1}  } s_i = -\g w_1^T \overline{\g x}_{2d+2}. 
\end{align*}
Thus, if $\g w_1^T \overline{\g x}_{2d+1} = y_{2d+1} = 0$ or $\g w_1^T \overline{\g x}_{2d+2} = y_{2d+2} = 0$, 
a valid partition in the sense of Problem~\ref{pb:partition} is obtained with $S_1 = \{s_i\}_{i\in \hat{I}_1}$. 
In addition, if $\g w_1^T \overline{\g x}_{2d+1} \neq  0$ and $\g w_1^T \overline{\g x}_{2d+2} \neq 0$, then by~\eqref{eq:perfectfit}, $\g w_{-1}^T \overline{\g x}_{2d+1} = \g w_{-1}^T \overline{\g x}_{2d+2} = 0$, which by construction implies that $w_{-1,d+1} = 0$. In this case, we redefine $\hat{I}_{-1} = \{ i\in\{1,\dots, d\} :  \g w_{-1}^T\overline{\g x}_i = y_i \}$ and $\hat{I}_{1} =  \{1,\dots, d\} \setminus \hat{I}_{-1}$ to obtain $w_{-1,i}=1$ for all $i\in \hat{I}_{-1}$ and $w_{-1,i}=-1$ for all $i\in \hat{I}_1$, resulting also in a valid partition by the fact that $\g w_{-1}^T \overline{\g x}_{2d+2} = \sum_{i\in \hat{I}_1} s_i - \sum_{i\in \hat{I}_{-1}} s_i = 0$. 
Since a similar reasoning applies to the case $h(\g x_{2d+3}) = -1$ by symmetry (substituting $\g w_{-1}$ for $\g w_1$), a zero cost, i.e., \eqref{eq:decision} with $\epsilon=0$, always implies an affirmative answer to Problem~\ref{pb:partition}.
\end{proof}

\begin{proof}[Proof of Theorem~\ref{thm:NPhard}]
Since the decision form of Problem~\ref{pb:PWA} with $\ell(e)=0 \Leftrightarrow e=0$, i.e., Problem~\ref{pb:PWAdecision}, is NP-complete, Problem~\ref{pb:PWA} with such a loss function is NP-hard (solving Problem~\ref{pb:PWA} also yields the answer to Problem~\ref{pb:PWAdecision} and thus it is at least as hard as Problem~\ref{pb:PWAdecision}). 
\end{proof}

\section{Polynomial complexity in the number of data}
\label{sec:exact}

We now state the result regarding the polynomial complexity of Problem~\ref{pb:PWA} with respect to $N$ under the following assumptions, the first of which holds almost surely for randomly drawn data points, while the second one holds for instance for $\ell(e) = e^2$ with a linear time complexity $T(N) = \O(N)$ \cite{Golub13}. 
\begin{assumption}\label{ass:generalposition}
The points $\{\g x_i\}_{i=1}^N$ are in general position, i.e., no hyperplane of $\QR^d$ contains more than $d$ points.
\end{assumption}
\begin{assumption}\label{ass:subpb}
Given $\{(\g x_i, y_i)\}_{i=1}^N \in (\X\times\QR)^N$, the problem $\min_{\g v\in\QR^{d+1} } \sum_{i=1}^N \ell (y_i - \g v^T \overline{\g x}_i)$ has a polynomial time complexity $T(N)$ for any fixed integer $d\geq 1$.
\end{assumption}
\begin{theorem}\label{thm:main}
For any fixed number of modes $n$ and dimension $d$, under Assumptions~\ref{ass:generalposition}--\ref{ass:subpb}, the time complexity of Problem~\ref{pb:PWA} is no more than polynomial in the number of data $N$ and in the order of $T(N)\O\left(N^{d n ( n -1) / 2} \right)$.
\end{theorem}

The proof of Theorem~\ref{thm:main} 
relies on the existence of exact algorithms with complexity polynomial in $N$ for the binary case ($n=2$, Proposition~\ref{prop:exact}) and the multi-class case ($n\geq 3$, Corollary~\ref{cor:multiclass}). 
These algorithms are based on a reduction of Problem~\ref{pb:PWA} to a combinatorial search in two steps. The first step reduces the problem to a classification one. Indeed, Problem~\ref{pb:PWA} can be reformulated as the search for the classifier $h$, since by fixing $h$, the optimal parameter vectors $\{\g w_j\}_{j\in\Q}$ can be obtained by solving $n$ independent linear regression problems on the subsets of data resulting from the classification by $h$, which, by Assumption~\ref{ass:subpb}, can be performed in the polynomial time $T(N)$. This yields the following reformulation of the problem.





\begin{proposition}\label{prop:pwals}
Problem~\ref{pb:PWA} is equivalent to
\begin{align}\label{eq:pwals}
&\min_{\g w\in\QR^{n(d+1)}, h\in\H }\ \frac{1}{N} \sum_{i=1}^N \sum_{j\in\Q} \I{h(\g x_i) = j}\ \ell\left(y_i- \g w_j^T\overline{\g x}_i\right) \\
&\mbox{s.t. }	\forall j\in\Q,\quad \g w_j \in \argmin_{\g v\in\QR^{d+1}} \ \sum_{i=1}^N \I{h(\g x_i) = j} \ell(y_i- \g v^T\overline{\g x}_i ).\nonumber
\end{align}
\end{proposition}

The second step reduces the estimation of $h$ to a combinatorial problem solved in $\O\left(N^{d n ( n -1) / 2} \right)$ operations, as detailed in Sect.~\ref{sec:optclassif}--\ref{sec:globalpwa} for $n=2$ and in Sect.~\ref{sec:multiclass} for $n\geq 3$.

\subsection{Finding the optimal classification}
\label{sec:optclassif}

We reduce the complexity of searching for the classifier by considering all possible linear classifications instead of all possible linear classifiers. In other words, we project the class $\H$ of classifiers onto the set of points $S=\{\g x_i\}_{i=1}^N$ to reduce a continuous search to a combinatorial problem. This is in line with the techniques used in statistical learning theory \cite{Vapnik98} for the different purpose of computing error bounds for infinite function classes. 
Thus, we introduce definitions from this field.

\begin{definition}[Projection onto a set] 
The projection of a set of classifiers $\H \subset \Q^\X$ onto $S=\{\g x_i\}_{i=1}^N$, denoted $\H_S$, is the set of all labelings of $S$ that can be produced by a classifier in $\H$: 
$$
	\H_S = \{ (h(\g x_1),\dots, h(\g x_N) )\ : \ h\in\H \} \subseteq \Q^N .
$$
\end{definition}

\begin{definition}[Growth function] The {\em growth function} $\Pi_{\H}(N)$ of $\H$ at $N$ is the maximal number of labelings of $N$ points that can be produced by classifiers from $\H$:
$$
	\Pi_{\H}(N) = \sup_{S \in \X^N} |\H_S|.
$$
\end{definition}

We now focus on binary PWA maps and thus on binary classifiers with output in $\Q=\{-1,+1\}$. For such classifiers, we obviously have $\Pi_{\H}(N) \leq 2^N$ for all $N$. By further restricting $\H$ to affine classifiers as in~\eqref{eq:Hbinary}, results from statistical learning theory (see, e.g., \cite{Vapnik98}) provide the tighter bound
%
%
$
	\Pi_{\H}(N) \leq \left( \frac{e N}{d+1}\right)^{d+1} 
$, which is polynomial in $N$ and thus promising from the viewpoint of global optimization. However, its proof is not constructive and does not provide an explicit algorithm for enumerating all the labelings. 
The following theorem, though leading to a looser bound on the growth function, offers a constructive scheme to compute the projection $\H_S$, which is what we need in order to test all the labelings in $\H_S$ for global optimization. 

\begin{theorem}\label{thm:enum}
The growth function of the class of binary affine classifiers of~$\QR^d$, $\H$ in \eqref{eq:Hbinary}, is bounded for any $N>d$ by
$$
	\Pi_{\H}(N) \leq 2^{d+1} \begin{pmatrix}N\\d\end{pmatrix} = \O(N^d)
$$
and, for any set $S$ of $N$ points in general position, an algorithm builds the projection $\H_S$ in $\O(N^d)$ time.
\end{theorem}



The proof of Theorem~\ref{thm:enum} relies on the following proposition, which is illustrated by Fig.~\ref{fig:prophyperplane}.

\begin{figure}
\centering
\includegraphics[width=0.6\linewidth]{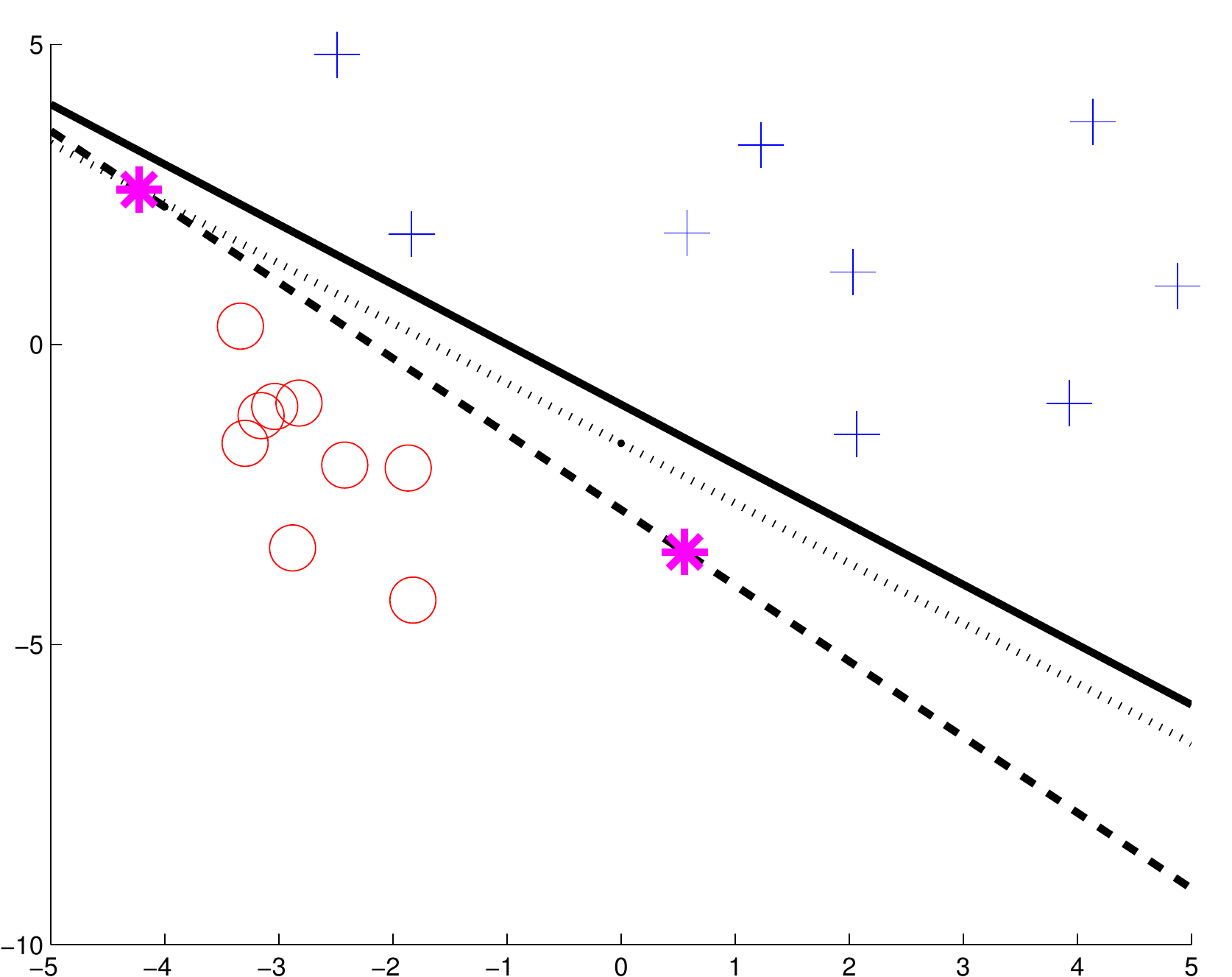}
\caption{The hyperplane $H$ (plain line) produces the same classification (into \textcolor{blue}{$\g +$} and \textcolor{red}{$\g\circ$})  as the hyperplane $H_S$ (dashed line) obtained by a translation (dotted line) and a rotation of $H$ such that it passes through exactly 2 points of $S$ 
(\textcolor{magenta}{$\g *$}). \label{fig:prophyperplane}}
\end{figure}



\begin{proposition}\label{prop:hyperplane}
For any binary affine classifier $h$ in $\H$ \eqref{eq:Hbinary}  
and any finite set of $N>d$ points $S = \{\g x_i\}_{i=1}^N$ in general position, there is a subset of points $S_h\subset S$ of cardinality $|S_h|=d$ and a separating hyperplane of parameters $(\g h_{S_h}, b_{S_h})$ passing through the points in $S_h$, i.e., 
\begin{equation}\label{eq:computehb}
	\forall \g x\in S_h,\quad \g h_{S_h}^T \g x + b_{S_h} = 0, \quad \mbox{with }\|\g h_{S_h}\| = 1, 
\end{equation}
which yields the same classification of $S$ in the sense that
\begin{align}\label{eq:equivclassif}
	\forall \g x_i \in S\setminus S_h,\quad h(\g x_i) = \sign( \g h_{S_h}^T \g x_i + b_{S_h}). 
\end{align}
\end{proposition}
\begin{proofsketch}
For all classifiers $h$ with separating hyperplanes passing through $d$ points of $S$, the statement is obvious. For the others passing through $p$ points with $0\leq p<d$, they can be transformed to pass through additional points without changing the classification of the remaining points. If $p=0$, it suffices to translate the hyperplane to the closest point. If $0<p<d$, the hyperplane can be rotated with a plane of rotation that leaves unchanged the subspace spanned by the $p$ points and a minimal angle yielding a rotated hyperplane passing through
$p^\prime > p$ points, where $p^\prime \leq d$ by the general position assumption. Iterating this scheme until $p=d$ 
yields a hyperplane passing through the points in $S_h$ of parameters $(\g h_{S_h}, b_{S_h})$ satisfying~\eqref{eq:computehb} and~\eqref{eq:equivclassif}. 
\end{proofsketch}

We can now prove Theorem~\ref{thm:enum}.

\begin{proof}[Proof of Theorem~\ref{thm:enum}]
For any labeling $\g q$ in $\H_S$, there is a classifier $h\in\H$ that produces this labeling. Applying Proposition~\ref{prop:hyperplane} to $h$, we obtain another classifier $h_{S_h}$ of parameters $(\g h_{S_h}, b_{S_h})$ that passes through the points in $S_h$ and that agrees with $h$ on $S\setminus S_h$. Let $\hat{\g q}\in\{-1,+1\}^N$ be defined by $\hat{q}_i = h_{S_h}(\g x_i)$, $i=1,\dots,N$. Then, we generate $2^d$ labelings by setting its entries $\hat{q}_{i}$ with $i\in S_h$ to all the $2^d$ combinations of signs (recall that $|S_h| = d$). By construction, there is no labeling of $S$ that agrees with $\g q$ on $S\setminus S_h$ other than these $2^d$ labelings.
Since this holds for any $\g q\in\H_S$, the cardinality of $\H_S$ cannot be larger than $2^d$ times the number of hyperplanes passing through $d$ points of $S$. Since each subset $S_h\subset S$ of cardinality $d$ gives rise to two hyperplanes of opposite orientations, the number of such hyperplanes is $2\begin{pmatrix}N\\d\end{pmatrix}$ and we have $\Pi_{\H}(N) \leq 2^{d+1}\begin{pmatrix}N\\d\end{pmatrix} <  2^{d+1}\frac{N^d}{d!} = \O(N^d)$. 
In addition, there is an algorithm that enumerates all the subsets $S_h$ in $\begin{pmatrix}N\\d\end{pmatrix}$ iterations and builds $\H_S$ by computing a hyperplane passing through the points\footnote{The normal $\g h$ of a hyperplane $\{\g x : \g h^T \g x + b = 0\}$ passing through $d$ points $\{\g x_i\}_{i=1}^d$ in $\QR^d$ can be computed as a unit vector in the null space of  $[\g x_2-\g x_1, \dots,$ $\g x_d - \g x_1]^T$, while the offset is given by $b= -\g h^T\g x_i$ for any of the $\g x_i$'s.} in $S_h$ and the corresponding $2^{d+1}$ labelings at each iteration. Since these inner computations can be performed in constant time with respect to $N$, the algorithm has a time complexity in the order of $\begin{pmatrix}N\\d\end{pmatrix}=\O(N^d)$.
\end{proof}

\subsection{Global optimization of binary PWA models}
\label{sec:globalpwa}

We can use the results above to reduce the complexity of Problem~\ref{pb:PWA} in the binary case, considered in the following in its equivalent form \eqref{eq:pwals} from Proposition~\ref{prop:pwals}. First, note that the cost function in~\eqref{eq:pwals} only depends on $h$, since all feasible values of $\g w$ for a given $h$ yield the same cost. Furthermore, the cost 
does not depend on the exact value of $h$, but only on the resulting classification, i.e., on $h(\g x_i)$, $i=1,\dots,N$. Thus, given a global solution $h^*$ to \eqref{eq:pwals}, any classifier $h$ producing the same classification yields the same cost function value and hence is also a global solution. Thus, the problem reduces to the search for the correct classification $\g q \in \H_S$, whose complexity is in $\O(\Pi_{\H}(N))$ and bounded by Theorem~\ref{thm:enum}. 
In addition, for the purpose of binary PWA regression, opposite labelings $\g q$ and $-\g q$ are equivalent and can be pruned from $\H_S$. This is due to the symmetry of the cost function \eqref{eq:pwals}. 
Algorithm~\ref{alg:exact} provides a solution to Problem~\ref{pb:PWA} for the binary case while taking this symmetry into account. 

\begin{algorithm}
\caption{Exact solution to Problem~\ref{pb:PWA} for $n=2$\label{alg:exact}}
\begin{algorithmic}
\REQUIRE A data set $\{(\g x_i,y_i)\}_{i=1}^N \subset (\QR^d\times \QR)^N$.
\STATE Initialize $S\leftarrow \{\g x_i\}_{i=1}^N$ and  $J^* \leftarrow +\infty$.
\FORALL{ $S_{h}\subset S$ such that $|S_h|=d$}
		\STATE Compute the parameters $(\g h_{S_h}, b_{S_h})$ of a hyperplane passing through the points in $S_h$.
		\STATE Classify the data points: $S_1 = \{\g x_i\in S\ : \ \g h_{S_h}^T \g x_i + b_{S_h} > 0\}$, $S_2 = \{\g x_i\in S\ : \ \g h_{S_h}^T \g x_i + b_{S_h} < 0\}$.
	\FORALL{ classification of $S_h$ into $S_h^1$ and $S_h^2$}
		\STATE Set 
		$\displaystyle{
			\g w_j \in \argmin_{\g v\in\QR^{d+1}}  \sum_{\g x_i \in S_j\cup S_h^j}\!\!\ell(y_i- \g v^T\overline{\g x}_i ) ,\quad j=1,2,
		}$
		\STATE 
		$\displaystyle{
		J = \frac{1}{N} \sum_{j=1}^2\ \sum_{\g x_i \in S_j\cup S_h^j} \ell(y_i- \g w_j^T\overline{\g x}_i ) ,
		}$
		\STATE and update the best solution $(J^*\leftarrow J, \g w^*\leftarrow [\g w_1^T, \g w_2^T]^T, \g h^* \leftarrow \g h_{S_h}, b^* \leftarrow b_{S_h} )$ if $J<J^*$. 
	\ENDFOR
\ENDFOR
\RETURN $\g w^*, \g h^*, b^*$.
\end{algorithmic}
\end{algorithm}

\begin{proposition}\label{prop:exact}
Under Assumptions~\ref{ass:generalposition}--\ref{ass:subpb}, Algorithm~\ref{alg:exact} exactly solves Problem~\ref{pb:PWA} for $n=2$ and any fixed $d$ with a polynomial complexity in the order of $T(N)\O(N^d)$. 
\end{proposition}
\begin{proof}
By following a similar path as for Theorem~\ref{thm:enum}, Algorithm~\ref{alg:exact} can be proved to test all linear classifications of the data points up to symmetric ones. Since Algorithm~\ref{alg:exact} computes a solution in terms of $\g w$ that is feasible for~\eqref{eq:pwals} for each of these classifications, the value of $J$ coincides with the cost function of~\eqref{eq:pwals} for a particular $h$. By the symmetry of this cost function with respect to $h$ and the fact that it only depends on $h$ via its values at the data points, Algorithm~\ref{alg:exact} computes all possible values of the cost function, including the exact global optimum of~\eqref{eq:pwals}, and returns a global minimizer. Thus, by Proposition~\ref{prop:pwals}, it also solves Problem~\ref{pb:PWA}.
The total number of iterations of Algorithm~\ref{alg:exact} is $2^d \begin{pmatrix}N\\d\end{pmatrix}= \O(N^d)$ 
and, under Assumption~\ref{ass:subpb}, these iterations only involve operations computed in polynomial time in the order of $T(N)$, hence the overall time complexity in the order of $T(N)\O(N^d)$.
\end{proof}

\subsection{Multi-class extension}
\label{sec:multiclass}

For $n>2$, the boundary between 2 modes $j$ and $k>j$ implemented by a linear classifier from $\H$ in \eqref{eq:hmulticlass} is a hyperplane of equation $h_{jk}(\g x) = h_j(\g x) - h_k(\g x) = 0$, i.e., based on the difference of the two functions $h_j(\g x) = \g h_j^T\g x + b_j$ and $h_k(\g x) = \g h_k^T\g x + b_k$. Based on these hyperplanes, the classification rule can be written as
\begin{align}\label{eq:decisionmulticlass}
 	h(\g x) &= \argmax_{k\in\mathcal{Q}} h_k(\g x)  
 	= j,\ \mbox{such that } \begin{cases} h_{jk}(\g x) \geq 0 ,\ \forall k> j,\\ h_{kj}(\g x) < 0,\ \forall k < j. \end{cases}	\nonumber
\end{align}
Based on these facts, we can build an algorithm to recover all possible classifications consistent with a linear classification in the sense of~\eqref{eq:hmulticlass}. 

\begin{theorem}\label{thm:multiclass}
For the set of  multi-class linear classifiers  of~$\QR^d$, $\H$ in \eqref{eq:hmulticlass}, the growth function is bounded for any $N>d$ by
$$
	\Pi_{\H} ( N) \leq \left[2^{d+1}\begin{pmatrix}N\\d\end{pmatrix}\right]^{n(n-1)/2} = \O(N^{dn(n-1)/2})
$$
and, for any set $S$ of $N$ points of $\QR^d$ in general position, an algorithm builds $\H_S$ in $\O(N^{dn(n-1)/2})$ time. 
\end{theorem}
\begin{proof}
Any classification produced by a classifier from \eqref{eq:hmulticlass} can be computed from the signs of the $n_H=n(n-1)/2$ functions $h_{jk} = h_j - h_k$, $1\leq j<k\leq n$, corresponding to the pairwise separating hyperplanes. For any $S$, for each of these hyperplanes, Proposition~\ref{prop:hyperplane} provides an equivalent binary classifier which must be one from the $2\begin{pmatrix}N\\d\end{pmatrix}$ hyperplanes passing through $d$ points $S_{jk}$ of $S$. The number of sets of $n_H$ such hyperplanes is $2^{n_H}\begin{pmatrix}N\\d\end{pmatrix}^{n_H}$. 
Since these classifiers cannot produce all the $2^{n_H d}$ classifications of the $n_H d$ points in the sets $S_{jk}$, we must also take these into account so that the number of classifications of $S$ is upper bounded by $|\H_S|\leq 2^{n_H d}2^{n_H}\begin{pmatrix}N\\d\end{pmatrix}^{n_H} = \O(N^{dn_H})$. This upper bound holds for any $S$, and thus also applies to the growth function. Finally, an algorithm that makes explicit all the classifications mentioned above to build $\H_S$ can be constructed in a recursive manner, with one classification per iteration and thus with a similar number of iterations, each one including  computations performed in constant time.
\end{proof}

Theorem~\ref{thm:multiclass} implies the following for PWA regression.

\begin{corollary}\label{cor:multiclass}
Under Assumptions~\ref{ass:generalposition}--\ref{ass:subpb}, a global solution to Problem~\ref{pb:PWA} with $n\geq 3$ can be computed with a polynomial complexity in the order of $T(N)\O(N^{dn(n-1)/2})$. 
\end{corollary}

\section{Conclusions}
\label{sec:discuss} 

The paper discussed complexity issues for PWA regression and showed that i) the global minimization of the error is NP-hard in general, and ii) for fixed number of modes and data dimension, an exact solution can be obtained in time polynomial in the number of data. 
The proof of NP-hardness also implies that the problem remains NP-hard even when the number of modes is fixed to~$2$, which indicates that the complexity is mostly due to the data dimension. 
An open issue concerns the conditions under which a PWA system generates trajectories satisfying the general position assumption used by the polynomial-time algorithm.  
Future work will also focus on the extension of the results to the case of arbitrarily switched systems and heuristics inspired by the polynomial-time algorithm, whose practical application remains limited by an exponential complexity in the dimension.

\section*{Acknowledgements}

The author would like to thank the anonymous reviewers for their comments and suggestions. Thanks are also due to Yann Guermeur for carefully reading this manuscript.



\end{document}